\documentclass{article}
\usepackage[utf8]{inputenc}
\usepackage{amsmath,amsfonts}
\usepackage{amsthm}
\usepackage{amssymb}
\usepackage{bm}
\usepackage{bbm}
\usepackage{natbib}
\usepackage[T1]{fontenc}    
\usepackage{hyperref}       
\usepackage{url}            
\usepackage{booktabs}       
\usepackage{amsfonts}       
\usepackage{nicefrac}       
\usepackage{microtype}      
\usepackage{float} 
\usepackage{color} 
\usepackage{algorithm}
\usepackage{algpseudocode}
\usepackage{graphicx}
\usepackage{fullpage}
\usepackage[T1]{fontenc}
\usepackage{lmodern}
\usepackage{xcolor}
\hypersetup{
	colorlinks,
	linkcolor={red!50!black},
	citecolor={blue!50!black},
	urlcolor={blue!80!black}
}

\newcommand{\todo}[1]{%
\ifmmode
\text{\textcolor{red}{TODO: #1}}
\else
\textcolor{red}{TODO: #1}
\fi
}

\DeclareMathOperator{\E}{\mathbb E}
\DeclareBoldMathCommand{\w}{w}
\DeclareBoldMathCommand{\u}{u}
\DeclareBoldMathCommand{\W}{W}
\DeclareBoldMathCommand{\U}{U}
\DeclareBoldMathCommand{\s}{s}
\DeclareBoldMathCommand{\x}{x}
\DeclareBoldMathCommand{\1}{1}
\DeclareBoldMathCommand{\0}{0}
\DeclareBoldMathCommand{\g}{g}
\DeclareBoldMathCommand{\z}{z}
\DeclareBoldMathCommand{\e}{e}
\DeclareBoldMathCommand{\p}{p}
\DeclareBoldMathCommand{\y}{y}

\newcommand{\reals}{\mathbb{R}}
\newcommand\doma[1]{\mathcal{#1}}
\newcommand{\regret}{\mathcal{R}}
\newcommand{\sumT}{\sum_{t=1}^T}

\newcommand{\sumk}{\sum_{k = 1}^K}
\newcommand{\sumd}{\sum_{i = 1}^d}

\newcommand{\ystart}{y_t^\star}
\newcommand{\mstart}{m_t^\star}

\newcommand{\yhatt}{\hat{y}_t}

\newcommand{\id}{\mathbbm{1}}
\newcommand{\ptstar}{p_t^\star}
\newcommand{\half}{\tfrac{1}{2}}
\newcommand\inner[2]{\langle #1, #2 \rangle}
\newcommand\Eb[1]{\E\left[ #1\right]}

\DeclareMathOperator*{\argmin}{arg\,min}
\DeclareMathOperator*{\argmax}{arg\,max}
\DeclareMathOperator{\sign}{sign}

\newtheorem{theorem}{Theorem}
\newtheorem{lemma}{Lemma}

\title{Exploiting the Surrogate Gap in Online Multiclass Classification}

\author{%
    Dirk van der Hoeven \\
    Leiden University \\
    dirk@dirkvanderhoeven.com
    }
\date{}

\begin{document}

\maketitle

\begin{abstract}
    We present \textproc{Gaptron}, a randomized first-order algorithm for online multiclass classification. In the full information setting we provide expected mistake bounds for \textproc{Gaptron} with respect to the logistic loss, hinge loss, and the smooth hinge loss with $O(K)$ regret, where the expectation is with respect to the learner's randomness and $K$ is the number of classes. In the bandit classification setting we show that \textproc{Gaptron} is the first linear time algorithm with $O(K\sqrt{T})$ expected regret.  Additionally, the expected mistake bound of \textproc{Gaptron} does not depend on the dimension of the feature vector, contrary to previous algorithms with $O(K\sqrt{T})$ regret in the bandit classification setting. We present a new proof technique that exploits the gap between the zero-one loss and surrogate losses rather than exploiting properties such as exp-concavity or mixability, which are traditionally used to prove logarithmic or constant regret bounds. 
\end{abstract}

\section{Introduction}\label{sec:introduction}

In online multiclass classification a learner has to repeatedly predict the label that corresponds to a feature vector. Algorithms in this setting have a wide range of applications ranging from predicting the outcomes of sport matches to recommender systems. In some applications such as sport forecasting the learner obtains the true label regardless of what outcome the learner predicts, but in other applications such as recommender systems the learner only learns whether or not the label he predicted was the true label. The setting in which the learner receives the true label is called the full information multiclass classification setting and the setting in which the learner only receives information about the predicted label is called the bandit multiclass classification setting. 

In this paper we consider both the full information and bandit multiclass classification settings. In both settings the environment chooses the true outcome $y_t \in \{1, \ldots, K\}$ and feature vector $\x_t \in \reals^d$. The environment then reveals the feature vector to the learner, after which the learner issues a (randomized) prediction $y_t' \in \{1, \ldots, K\}$. The goal of both settings is to minimize the number of expected mistakes the learner makes with respect to the best offline linear predictor $\U \in \reals^{K \times d}$, which essentially contains a set of parameters per class. Standard practice in both settings is to upper bound the non-convex zero-one loss with a convex surrogate loss $\ell_t$ (see for example \citet{bartlett2006convexity}). This leads to guarantees of the form
\begin{align*}
    \E\left[\sumT \id[y_t' \not = y_t]\right] = \E\left[\sumT \ell_t(\U)\right] + \regret_T,
\end{align*}
where $\id$ is the indicator function, $y_t$ is the true label, the expectation is taken with respect to the learner's randomness, and $\regret_T$ is the regret after $T$ rounds. 

We introduce \textproc{Gaptron}, which is a randomized first-order algorithm that exploits the gap between the zero-one loss and the surrogate loss. In the full information multiclass classification setting \textproc{Gaptron} has $O(K)$ regret with respect to various surrogate losses. In the bandit multiclass classification setting we show that \textproc{Gaptron} has $O(K\sqrt{T})$ regret with respect to the same surrogate losses as in the full information setting. Importantly, our regret bounds do not depend on the dimension of the feature vector in either the full or bandit information setting, contrary to previous results with similar regret bounds. Furthermore, in the bandit multiclass classification setting \textproc{Gaptron} is the first $O(dK)$ running time algorithm with $O(K\sqrt{T})$ regret.

To achieve these results we develop a new proof technique. Standard approaches that lead to small regret bounds exploit properties of the surrogate loss function such as strong convexity, exp-concavity \citep{hazan2007logarithmic}, or mixability \citep{vovk2001competitive}. Instead, inspired by the recent success of \citet{neu2020fast} in online classification with abstention\footnote{In fact, in Appendix \ref{app:adahedge with abstentions} we slightly generalize the results of \citet{neu2020fast}.}, we exploit the \emph{gap} between the zero-one loss, which is used to measure the performance of the learner, and the surrogate loss, which is used to measure the performance of the comparator $\U$, hence the name \textproc{Gaptron}. 

For an overview of our results and a comparison to previous work see Table \ref{tab:overview}. Here we briefly discuss the most relevant literature to place our results into perspective. A more detailed comparison can be found in the relevant sections. The full information multiclass classification setting is well understood and has been studied by many authors. Perhaps the most well known algorithm in this setting is the \textproc{Perceptron} \citep{rosenblatt1958perceptron} and its multiclass versions \citep{crammer2003ultraconservative, fink2006online}. The \textproc{Perceptron} is a deterministic first-order algorithm which has $O(\sqrt{T})$ regret with respect to the hinge loss in the worst-case. Variants of the \textproc{Perceptron} such as \textproc{Arow} \citep{crammer2009adaptive} and the second-order \textproc{Perceptron} \citep{cesa2005second} are second-order methods which result in  a possibly smaller regret at the cost of longer running time. Online logistic regression \citep{berkson1944application} is an alternative to the \textproc{Perceptron} which has been thoroughly studied. For an overview of results for online logistic regression we refer the reader to \citet{shamir2020logistic}. Here we mention a recent result by \citet{foster2018logistic}, who use Exponential Weights \citep{vovk1990aggregating, LittlestoneWarmuth1994} to optimize the logistic loss and obtain a regret bound of $O(dK\ln(DT+1))$, where $D$ is an upper bound on the Frobenius norm of $\U$, with a polynomial time algorithm.

The bandit multiclass classification setting was first studied by \citet{kakade2008efficient} and is a special case of the contextual bandit setting \citep{langford2008epoch}. \citet{kakade2008efficient} present a first-order algorithm called \textproc{Banditron} with a $O((DK)^{1/3}T^{2/3})$ regret bound with respect to the hinge loss. The impractical \textproc{EXP4} algorithm \citep{auer2002nonstochastic} has a $O(\sqrt{TdK\ln(T+1)})$ regret bound and \citet{abernethy2009efficient} posed the problem of obtaining a practical algorithm which attains an $O(K\sqrt{T})$ regret bound. Several authors have proposed polynomial running time algorithms that have a regret bound of $O(K\sqrt{d T\log(T+1)})$ such as \textproc{Newtron} \citep{hazan2011newtron}, \textproc{Confidit} \citep{crammer2009adaptive},  \textproc{Soba} \citep{beygelzimer2017efficient}, and \textproc{Obama} \citep{foster2018logistic}.

\begin{table}[t]
\caption{Main results and comparisons with previous work (see Section \ref{sec:prelim} for notation). The references are for the regret bounds, not necessarily for the first analysis of the algorithm. For this table we assume that $\|\x_t\| \leq 1 ~\forall t$ and denote by $L_T = \sumT\ell_t(\U)$ the sum of the surrogate losses of the comparator.}
\centering
\label{tab:overview}
\resizebox{\textwidth}{!}{\begin{tabular}{p{40mm}p{20mm}ccc}
  \hline
Algorithm  & Loss & Regret full information setting & Regret bandit setting & Time (per round)\\ 
  \hline
  \hline
  \textproc{Perceptron}\citep{fink2006online, kakade2008efficient}  & hinge & $O(\|\U\|^2 + \|\U\|\sqrt{L_T})$ & $O((D K)^{1/3}T^{2/3})$ & $O(dK)$ \\
  \hline
  Second-Order \textproc{Perceptron} \citep{orabona2012beyond, beygelzimer2017efficient} & hinge\footnotemark & $O(\frac{\kappa}{2 - \kappa}\|\U\|^2 + \frac{dK}{\kappa(2 - \kappa)} \ln(L_T) )$ & $O(\|\U\|^2 + \frac{K}{\kappa}\sqrt{dT\ln(T)})$ & $O((dK)^2)$\\
  \hline
  \textproc{ONS} \citep{hazan2014logistic, hazan2011newtron} & logistic & $O(\exp(D)dK\ln(T))$ & $O(d K^3 D T^{2/3})$ & $O((dK)^2)$ \\
  \hline
  Vovk's Aggregating Algorithm \citep{foster2018logistic} & logistic  & $O(dK\ln(DT))$ & $O(K\sqrt{dT\ln(DT)})$ & $O(\max\{dK, T\}^{12})$ \\
  \hline
  \textproc{Gaptron} (This work) & logistic, hinge, smooth hinge  & $O(K \|\U\|^2)$ & $O(KD\sqrt{T})$ & $O(dK)$ \\
  \hline
\end{tabular}}
\end{table}
\footnotetext{These results hold for a family of loss functions parametrized by $\kappa \in [0, 1]$, which includes the hinge loss.}

\section{Preliminaries}\label{sec:prelim}

\paragraph{Notation.} Let $\1$ and $\0$ denote vectors with only ones and zeros respectively and let $\e_k$ denote the basis vector in direction $k$. The inner product between vectors $\g \in \reals^d$ and $\w \in \reals^d$ is denoted by $\langle \w, \g \rangle$. The rows of matrix $\W \in \reals^{K \times d}$ are denoted by $\W^1, \ldots, \W^K$. We will interchangeably use $\W$ to denote a matrix and a column vector in $\reals^{Kd}$ to avoid unnecessary notation. The vector form of matrix $\W$ is $(\W^1, \ldots, \W^K)^\top$. The Frobenius norm of matrix $\W$ is denoted by $\|\W\| = \sqrt{\sumk \sum_{i = 1}^d W_{k, i}^2}$. Likewise the $l_2$ norm of vector $\x$ is denoted by $\|\x\| = \sqrt{\sumd x_i^2}$. We denote the Kronecker product between matrices $\W$ and $\U$ by $\W \otimes \U$. For a given round $t$ we use $\E_t[\cdot]$ to denote the conditional expectation given the predictions $y_1', y_2', \ldots, y_{t-1}'$.

\subsection{Multiclass Classification}

The multiclass classification setting proceeds in rounds $t = 1, \ldots, T$. In each round $t$ the environment first picks an outcome $y_t \in \{1, \ldots K\}$ and feature vector $\x_t$ such that $\|\x_t\| \leq X$ for all $t$. Before the learner makes his prediction $y_t'$ the environment reveals the feature vector $\x_t$ which the learner may use to form $y_t'$. In the full information multiclass classification setting, after the learner has issued $y_t'$, the environment reveals the outcome $y_t$ to the learner. In the bandit multiclass classification setting \citep{kakade2008efficient} the environment only reveals whether the prediction of the learner was correct or not, i.e. $\id[y_t' \not = y_t]$. We only consider the adversarial setting, which means that we make no assumptions on how $y_t$ or $\x_t$ is generated. In both settings we allow the learner to use randomized predictions. The goal of the multiclass classification setting is to control the number of expected mistakes the learner makes in $T$ rounds: $M_T = \E\left[\sumT \id[y_t' \not = y_t]\right]$, where the expectation is taken with respect to the learner's randomness. 

Since the zero-one loss is non-convex a standard approach is to use a surrogate loss $\ell_t$ as a function of a weight matrix $\W_t \in \doma{W}$, where $\doma{W} = \{\W: \|\W\| \leq D\}$. The surrogate loss function is a convex upper bound on the zero-one loss, which is then optimized using an Online Convex Optimization algorithm such as Online Gradient Descent (OGD) \citep{zinkevich2003}, Online Newton Step (ONS) \citep{hazan2007logarithmic}, or Exponential Weights (EW) \citep{vovk1990aggregating, LittlestoneWarmuth1994}. 
In this paper we treat three surrogate loss functions: logistic loss, the hinge loss, and the smooth hinge loss, all of which result in different guarantees on the number of expected mistakes a learner makes.

\section{\textproc{Gaptron}} \label{sec:gap}

\begin{algorithm}[t]
\caption{\textproc{Gaptron}}\label{alg:gapjumper}
\begin{algorithmic}[1]
\Require Learning rate $\eta > 0$, exploration rate $\gamma \in [0, 1]$, and gap map $a:\reals^{K \times d} \times \reals^d \rightarrow [0, 1]$ \\
\textbf{Initialize} $\W_1 = \0$
\For{$t = 1 \ldots T$}
\State Obtain $\x_t$ 
\State Let $\ystart = \argmax_{k} \inner{\W_t^k}{\x_t}$
\State Set $\p_t' = (1 - \max\{a(\W_t, \x_t), \gamma \}) \e_{\ystart} + \max\{a(\W_t, \x_t), \gamma \} \frac{1}{K}\1$
\State Predict with label $y_t' \sim \p_t'$
\State Obtain $\id[y_t' \not = y_t]$ and set $\g_t = \nabla \ell_t(\W_t)$
\State Update $\W_{t+1} = \argmin_{\W \in \doma{W}} \eta \inner{\g_t}{\W} + \frac{1}{2}\|\W - \W_t\|^2$
\EndFor
\end{algorithmic}
\end{algorithm}

In this section we discuss \textproc{Gaptron} (Algorithm \ref{alg:gapjumper}). The prediction $y_t'$ is sampled from 
\begin{equation*}
    \p_t' = (1 - \max\{a(\W_t, \x_t), \gamma \}) \e_{\ystart} + \max\{a(\W_t, \x_t), \gamma \} \frac{1}{K}\1, 
\end{equation*}
where $\gamma \in [0, 1]$ and $a:\reals^{K \times d} \times \reals^d \rightarrow [0, 1]$ to be specified later. In the full information setting $\gamma$ is set to 0 but in the bandit setting $\gamma$ is used to guarantee that each label is sampled with at least probability $\frac{\gamma}{K}$, which is a common strategy in bandit algorithms (see for example \citet{auer2002nonstochastic}). The fact that each label is sampled with at least probability $\frac{\gamma}{K}$ is important because in the bandit setting we use importance weighting to form estimated surrogate losses $\ell_t$ and their gradients $\g_t = \nabla \ell_t(\W_t)$ and we need to control the variance of these estimates. The main difference between \textproc{Gaptron} and standard algorithms for multiclass classification is the $a$ function, which governs the mixture that forms $\p_t'$. In fact, if we set $a(\W, \x) = 0$, $\gamma = 0$, and choose $\ell_t$ to be the hinge loss we recover an algorithm that closely resembles the \textproc{Perceptron} \citep{rosenblatt1958perceptron}, which can be interpreted as OGD on the hinge loss\footnote{Other interpretations exist which lead to possibly better guarantees, see for example \cite{beygelzimer2017efficient}.}. \textproc{Gaptron} also uses OGD, which is used to update weight matrix $\W_t$, which in turn is used to form distribution $\p_t'$. For convenience we will define $a_t = a(\W_t, \x_t)$. 

The role of $a$, which we will refer to as the gap map, is to exploit the gap between the surrogate loss and the zero-one loss. Before we explain how we exploit said gap we first present the expected mistake bound of \textproc{Gaptron} in Lemma \ref{lem: surrogate gap}. The proof of Lemma \ref{lem: surrogate gap} follows from applying the regret bound of OGD and working out the expected number of mistakes. The formal proof can be found in Appendix \ref{app:gap}.

\begin{lemma}\label{lem: surrogate gap}
For any $\U \in \doma{W}$ Algorithm \ref{alg:gapjumper} satisfies
\begin{align*}
    \E\left[\sumT \id[y_t' \not = y_t]\right] \leq &  \E\left[\sumT \ell_t(\U)\right] + \frac{\|\U\|^2}{2\eta} + \gamma \frac{K-1}{K} T  \\
     & + \sumT \underbrace{\E\left[(1 - a_t)\id[\ystart \not = y_t] + a_t \frac{K-1}{K} - \ell_t(\W_t) + \frac{\eta}{2}\|\g_t\|^2\right]}_{\textnormal{surrogate gap}}.
\end{align*}
\end{lemma}

As we mentioned before, standard classifiers such as the \textproc{Perceptron} simply set $a(\W, \x) = 0$ and upper bound $\id[\ystart \not = y_t] - \ell_t(\W_t)$ by 0. In the full information setting we can set $\gamma = 0$ and $\eta = \sqrt{\frac{\|\U\|^2}{\sumT \|\g_t\|^2}}$ to obtain\footnote{Although such tuning is impossible due to not knowing $\|\U\|$ or $\sumT \|\g_t\|^2$ there exist algorithms that are able to achieve the same guarantee up to logarithmic factors, see for example \citet{cutkosky2018black}.} $M_T \leq \sumT \ell_t(\U) + \|\U\|\sqrt{\sumT \|\g_t\|^2}$. However, the gap between the surrogate loss and the zero-one loss can be large. In fact, even with $a(\W, \x) = 0$, the gap between the zero-one loss and the surrogate loss is large enough to bound $\id[\ystart \not = y_t] - \ell_t(\W_t) + \frac{\eta}{2} \|\g_t\|^2$ by 0 for some loss functions and values of $\W_t$ and $\x_t$.

\begin{figure}[t]
    \centering
    \includegraphics[width=\textwidth, height=6cm]{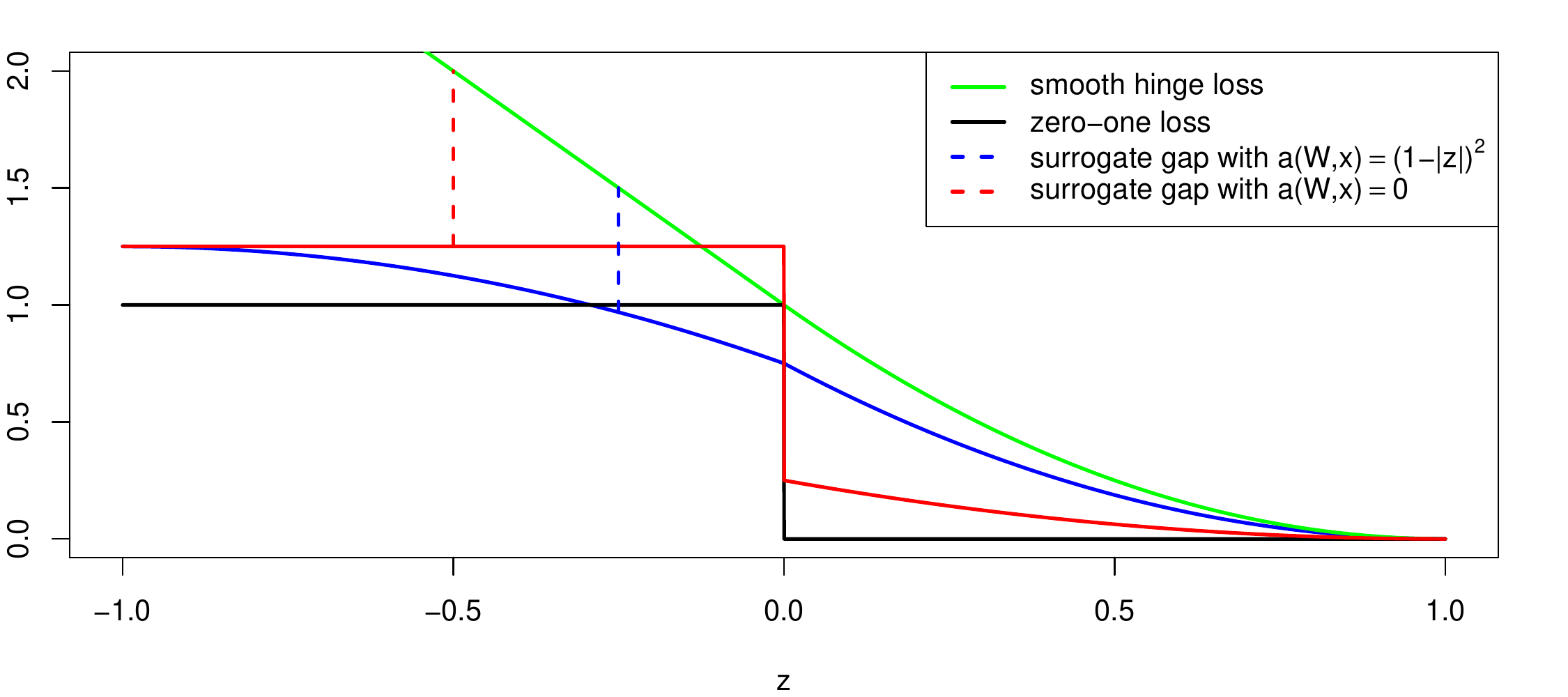}
    \caption{The surrogate gap for the smooth hinge loss as a function of margin $z$ with $\eta = \frac{1}{8}$, $\gamma = 0$, and $\|\x\| = 1$. The solid red line is given by $\id[z \leq 0] + \frac{\eta}{2}\|\g\|^2$, where $\|\g\|^2 = 4 (1 - z)^2$ if $z > 0$ and $\|\g\|^2 = 4$ otherwise. The solid blue line is given by $(1 - (1-|z|)^2)\id[z \leq 0] + \half (1-|z|)^2 + \frac{\eta}{2}\|\g\|^2$. The surrogate gap is positive whenever the red or blue line is above the green line.}
    \label{fig:smooth hinge gap}
\end{figure}

In Figure \ref{fig:smooth hinge gap} we can see a depiction of the surrogate gap for the smooth hinge loss for $K = 2$ \citep{rennie2005loss} in the full information setting (see Section \ref{sec: full smooth hinge} for the definition of the smooth multiclass  hinge loss). In the case where $K = 2$, $\W$ is a vector rather than a matrix and outcomes $y_t$ are coded as $\{-1, 1\}$. We see that with $a(W, \x) = 0$, only when margin $z = y \inner{\W}{\x} \in [-0.125, 0]$ the surrogate gap is not upper bounded by 0. Decreasing $\eta$ would increase the range for which the surrogate gap is bounded by zero, but only for $\eta = 0$ the surrogate gap is bounded by 0 everywhere. However, with $a(W, \x) = (1 - |z|)^2$ the surrogate gap is upper bounded by 0 for all $z$, which leads to $O(K)$ regret. The remainder of the paper is concerned with deriving different $a$ for different loss functions for which the surrogate gap is bounded by 0. For an overview of the different settings of \textproc{Gaptron} see Table \ref{tab:settings} in Appendix \ref{app:settings}. In the following section we start in the full information setting.

\section{Full Information Multiclass Classification} \label{sec:full info}

In this section we derive gap maps that allow us to upper bound the surrogate gap by 0 for the logistic loss, hinge loss, and smooth hinge loss in the full information setting. Note that the full information algorithms we compare with are deterministic whereas \textproc{Gaptron} is randomized. Throughout this section we will set $\gamma = 0$. We start with the result for logistic loss. 

\subsection{Logistic Loss}\label{sec: full logistic}

The logistic loss is defined as
\begin{align}\label{eq:full logistic loss}
    \ell_t(\W) = & -\log_2(\sigma(\W, \x_t, y_t)),
\end{align}
where $\sigma(\W, \x, k) = \frac{\exp(\inner{\W^k}{\x})}{\sum_{\tilde{k} = 1}^K \exp(\inner{\W^{\tilde{k}}}{\x})}$ is the softmax function. For the logistic loss we will use the following gap map:
\begin{align*}
    a(\W_t, \x_t) = 1- \id[\ptstar \geq 0.5]\ptstar,
\end{align*}
where $\ptstar = \max_k \sigma(\W_t, \x_t, k)$. This means that \textproc{Gaptron} samples a label uniformly at random as long as $\ptstar \leq 0.5$. While this may appear counter-intuitive at first sight note that when $\ptstar < 0.5$ the zero-one loss is upper bounded by the logistic loss regardless of what we play since $-\log_2(p) \geq 1$ for $p \in [0, 0.5]$, which we use to show that the surrogate gap is bounded by 0 whenever $\ptstar < 0.5$. The mistake bound of $\textproc{Gaptron}$ can be found in Theorem \ref{th:full logistic}. To prove Theorem \ref{th:full logistic} we show that the surrogate gap is bounded by 0 and then use Lemma \ref{lem: surrogate gap}. The formal proof can be found in Appendix \ref{app:full logistic}.

\begin{theorem}\label{th:full logistic}
Let $a(\W_t, \x_t) = 1- \id[\ptstar \geq 0.5]\ptstar$, $\eta = \frac{\ln(2)}{2 K X^2}$, $\gamma = 0$, and let $\ell_t$ be the logistic loss defined in \eqref{eq:full logistic loss}. Then for any $\U \in \doma{W}$ Algorithm \ref{alg:gapjumper} satisfies
\begin{align*}
    \E\left[\sumT  \id[y_t' \not = y_t]\right] \leq & \sumT \ell_t(\U) + \frac{K X^2 \|\U\|^2}{\ln(2)}.
\end{align*}
\end{theorem}

Let us compare the mistake bound of \textproc{Gaptron} with other results for logistic loss. \citet{foster2018logistic} circumvent a lower bound for online logistic regression by \citet{hazan2014logistic}  by using an improper learning algorithm and achieve $O(dK\ln(DT + 1))$ regret. Unfortunately this algorithm is impractical since the running time can be of $O(D^6\max\{dK, T\}^{12} T)$. In the case where $K = 2$ \citet{jezequel2020efficient} provide a faster improper learning algorithm called \textproc{Aioli} based on the Vovk-Azoury-Warmuth forecaster \citep{vovk2001competitive, azoury2001relative} that has running time $O(d^2T)$ and a regret of $O(dD\ln(T))$. Unfortunately it is not known if \textproc{Aioli} can be extended to $K > 2$. An alternative algorithm is ONS, which has running time $O((dK)^2 T)$ but a regret bound of $O(\exp(D)dK \ln(T+1))$. With standard OGD we could degrade the dependence on $T$ to improve the dependence on $D$ to find a regret of $O(D\sqrt{T})$ with an algorithm that has running time $O(dKT)$. Depending on $\|\U\|^2$ the regret of \textproc{Gaptron} can be significantly smaller than the regret of the aforementioned algorithms as the regret of \textproc{Gaptron} is independent of $T$ and $d$. Furthermore, since \textproc{Gaptron} uses OGD to update $\W_t$ the running time is $O(dKT)$, significantly improving upon the running time of previous algorithms with comparable mistake bounds.

\subsection{Multiclass Hinge Loss}\label{sec: full hinge}

We use a variant of the multiclass hinge loss of \citet{crammer2001algorithmic}, which is defined as: 
\begin{equation}\label{eq:multiclass hinge}
    \ell_t(\W) = 
    \begin{cases}
        \max\{1-m_t(\W, y_t), 0\} & \text{ if } \mstart \leq \beta  \\
        \max\{1-m_t(\W, y_t), 0\} & \text{ if } \ystart \not = y_t \text{ and } \mstart > \beta  \\
        0 & \text{ if } \ystart = y_t \text{ and } \mstart > \beta,
    \end{cases}
\end{equation}
where $m_t(\W_t, y) = \inner{\W_t^y}{\x_t} - \max_{k \not= y} \inner{\W_t^k}{\x_t}$ and $\mstart = \max_k m_t(\W_t, k)$. Note that we set $\ell_t(\W) = 0$ when $\ystart = y_t \text{ and } \mstart > \beta$. In common implementations of the \textproc{Perceptron} $\ell_t(\W) = 0$ whenever $\ystart = y_t$ (see for example \citet{kakade2008efficient}). However, for the surrogate gap to be bounded by zero we need $\ell_t$ to be positive whenever $a_t > 0$ otherwise there is nothing to cancel out the $a_t\frac{K-1}{K}$ term. The gap map for the hinge loss is $a(W_t, \x_t) = 1 - \max\{\id[\mstart > \beta], \mstart\}$. This means that whenever $\mstart > \beta$ the predictions of \textproc{Gaptron} are identical to the predictions of the \textproc{Perceptron}. The mistake bound of \textproc{Gaptron} for the hinge loss can be found in Theorem \ref{th: full hinge} (its proof is deferred to Appendix \ref{app:full hinge}).

\begin{theorem}\label{th: full hinge}
Set $a(\W_t, \x_t) = 1 - \max\{\id[\mstart > \beta], \mstart\}$, $\eta = \frac{1-\beta}{K X^2}$, $\gamma = 0$, and let $\ell_t$ be the multiclass hinge loss defined in \eqref{eq:multiclass hinge} with $\beta = \frac{1}{K}$. Then for any $\U \in \doma{W}$ Algorithm \ref{alg:gapjumper} satisfies
\begin{align*}
    \E\left[\sumT  \id[y_t' \not = y_t]\right] \leq & \sumT \ell_t(\U) +  \frac{K^2X^2\|\U\|^2}{2(K-1)}.
\end{align*}
\end{theorem}

Let us compare the mistake bound of \textproc{Gaptron} with the mistake bound of the \textproc{Perceptron}. The \textproc{Perceptron} guarantees $M_T \leq \sumT \ell_t(\U) + X^2\|\U\|^2 + 2X\|\U\|\sqrt{2\sumT \ell_t(\U)}$ (see \citet{beygelzimer2017efficient} for a proof). The factor $K$ in the regret of \textproc{Gaptron} is due to the cost of exploring uniformly at random. For small $K$ the mistake bound of \textproc{Gaptron} can be significantly smaller in the adversarial case, but for large $K$ the cost of sampling uniformly at random can be too high and the mistake bound of \textproc{Gaptron} can be larger than that of the \textproc{Perceptron}. In the separable case the \textproc{Perceptron} has a strictly better guarantee for any $K$ since then only the $X^2\|\U\|^2$ term remains. 

\citet{orabona2012beyond} show that for all loss functions of the form $\ell_t(\W) = \max\{1 - \frac{2}{2-\kappa}m_t(W, y_t) + \frac{\kappa}{2 - \kappa}m_t(W, y_t)^2, 0\}$ the second-order \textproc{Perceptron} guarantees $M_T \leq \sumT \ell_t(\U) + O(\frac{\kappa}{2-\kappa} X^2\|\U\|^2 + \frac{d K}{\kappa(2-\kappa)} \ln(\sumT \ell_t(\U) + 1)$. Thus, for small $K$ \textproc{Gaptron} always has a smaller regret term but for larger $K$ the guarantee of \textproc{Gaptron} can be worse, although this also depends on the performance and norm of the comparator $\U$.

\subsection{Smooth Multiclass Hinge Loss} \label{sec: full smooth hinge}
The smooth multiclass hinge loss \citep{rennie2005loss} is defined as
\begin{equation}\label{eq:multiclass smooth hinge}
    \ell_t(\W) = 
    \begin{cases}
        \max\{1 - 2 m_t(\W, y_t), 0\} & \text{ if }  m_t(\W, y_t) \leq 0  \\
        \max\{(1 - m_t(\W, y_t))^2, 0\} & \text{ if } m_t(\W, y_t) > 0,
        \end{cases}
\end{equation}
where $m_t(\W_t, y) = \inner{\W_t^y}{\x_t} - \max_{k \not= y} \inner{\W_t^k}{\x_t}$ as in Section \ref{sec: full smooth hinge}. This loss function is not exp-concave nor is it strongly-convex. This means that with standard methods from Online Convex Optimization we cannot hope to achieve a better regret bound than $O(D\sqrt{T})$ in the worst-case. Theorem \ref{th: full smooth hinge} shows that with gap map $a(\W_t, x_t) = (1 - \min\{1, \mstart\})^2$, where $\mstart = \max_k m_t(\W_t, k)$, \textproc{Gaptron} has an $O(K)$ regret bound. The proof of Theorem \ref{th: full smooth hinge} follows from bounding the surrogate gap by zero and can be found in Appendix \ref{app:full smooth hinge}.

\begin{theorem}\label{th: full smooth hinge}
Set $a(\W_t, \x_t) = (1 - \min\{1, \mstart\})^2$, $\eta = \frac{1}{4 K X^2}$, $\gamma = 0$, and let $\ell_t$ be the smooth multiclass hinge loss defined in \eqref{eq:multiclass smooth hinge}. Then for any $\U \in \doma{W}$ Algorithm \ref{alg:gapjumper} satisfies
\begin{align*}
    \E\left[\sumT  \id[y_t' \not = y_t]\right] \leq & \sumT \ell_t(\U)  + 2 K X^2\|\U\|^2.
\end{align*}
\end{theorem}

\section{Bandit Multiclass Classification}\label{sec: bandit}

In this section we will analyse \textproc{Gaptron} in the bandit multiclass classification setting. While in the full information setting the fact that \textproc{Gaptron} is a randomized algorithm can be seen as a drawback, in the adversarial bandit setting it is actually a requirement (see for example chapter 11 by \citet{lattimore2018bandit}). We will use the same gap maps as in the full information setting. The only difference is how we feed the surrogate loss to \textproc{Gaptron}. We will use the same loss functions as in the full information setting but now multiplied by $\frac{\id[y_t'= y_t]}{p_t'(y_t')}$, which is simply importance weighting. This also means that, compared to the full information setting, the gradients that OGD uses to update weight matrix $\W_t$ are multiplied by $\frac{\id[y_t'= y_t]}{p_t'(y_t')}$. To control the surrogate gap we set $\gamma > 0$, which allows us to bound the variance of the norm of the gradients. The proofs in this section follow the same structure as in the full information setting, with the notable change that we suffer increased regret due to the $\gamma \frac{K-1}{K}T$ bias term and the increased $\E[\|\g_t\|^2] = O(\frac{K}{\gamma})$ term.

The results in this section provide three new answers to the open problem by \citet{abernethy2009efficient}, who posed the problem of obtaining an efficient algorithm with $O(K\sqrt{T})$ regret. Several solutions with various loss function have been proposed. \citet{beygelzimer2017efficient} solved the open problem using an algorithm called \textproc{Soba}. \textproc{Soba} is a second-order algorithm which is analysed using a family of surrogate loss functions introduced by \citet{orabona2012beyond} ranging from the standard multiclass hinge loss to the squared multiclass hinge loss. The loss functions are parameterized by $\kappa$, where $\kappa = 0$ corresponds to the multiclass hinge loss and $\kappa = 1$ corresponds to the squared hinge loss. Simultaneously for all surrogate loss functions in the family of loss functions \textproc{Soba} suffers a regret of $O(\|\U\|^2X^2 + \frac{K}{\kappa}\sqrt{dT\ln(T+1)})$ and has a running time of order $O((dK)^2T)$. \citet{hazan2011newtron} consider the logistic loss and obtain regret of $O(dK^3\min\{\exp(DX)\ln(T+1), DXT^{\frac{2}{3}}\})$. \citet{hazan2011newtron} also obtain $DX\sqrt{T}$ regret for a variant of the logistic loss function we consider in this paper. Both results of \citet{hazan2011newtron} are obtained by running ONS on (a variant of) the logistic loss, which has running time $O((dK)^2 T)$. \citet{crammer2009adaptive} assume that a particular probabilistic model generates $y_t$ and obtain $O(DK\sqrt{dT}\log(T))$ expected regret bound with high probability, but for a sharper notion of regret. \citet{foster2018logistic} introduce \textproc{Obama}, which improves the results of \citet{hazan2011newtron} and suffers $O(\min\{dK^2\ln\left(TDX +1 \right), K\sqrt{dT\ln\left(TDX +1 \right)}\})$ regret for the logistic loss. Unfortunately, \textproc{Obama} has running time $O(D^6\max\{dK, T\}^{12} T)$. 

\textproc{Gaptron} is the first $O(dKT)$ running time algorithm which has $O(DK\sqrt{T})$ regret in bandit multiclass classification with respect to the logistic, hinge, or smooth hinge loss. \textproc{Gaptron} also improves the regret bounds of previous algorithms with $O(DK\sqrt{T})$ regret by a factor $O(\sqrt{d\log(T+1)})$. The remainder of this section provides the settings for \textproc{Gaptron} to achieve these results, starting with the logistic loss. 

\subsection{Bandit Logistic Loss}\label{sec: bandit logistic}

The bandit version of the logistic loss is defined as:
\begin{align}\label{eq:bandit logistic loss}
    \ell_t(\W) = -\frac{\id[y_t'= y_t]}{p_t'(y_t')}\log_2(\sigma(\W, \x_t, y_t)).
\end{align}
A similar definition of the bandit logistic loss is used by \citet{hazan2011newtron, foster2018logistic}. It is straightforward to verify that $\E_t[\ell_t(\w)]$ is equivalent to its full information counterpart \eqref{eq:full logistic loss}. This loss is a factor $\frac{1}{\ln(2)}$ larger than the loss used by \citet{hazan2011newtron, foster2018logistic}, who use the natural logarithm instead of the logarithm with base 2. To stay consistent with the full information setting we opt to use base 2 in the bandit setting. Using \textproc{Gaptron} with the natural logarithm will give similar results.

The mistake bound of \textproc{Gaptron} for this loss can be found in Theorem \ref{th:bandit logistic} (its proof can be found in Appendix \ref{app:bandit logistic}). Compared to \textproc{Obama}, which achieves a regret bound of $O(\min\{dK^2\ln\left(TDX +1 \right), K\sqrt{dT\ln\left(TDX +1 \right)}\})$, \textproc{Gaptron} has a larger dependency on $D$ and $X$. However, the mistake bound of \textproc{Gaptron} does not depend on $d$, which can be a significant improvement over the regret bound of \textproc{Obama}. Theorem \ref{th:bandit logistic} answers the two questions by \citet{hazan2011newtron} affirmatively; \textproc{Gaptron} is a linear time algorithm with exponentialy improved constants in the regret bound compared to \textproc{Newtron}.

\begin{theorem}\label{th:bandit logistic}
Let $a(\W_t, \x_t) = 1- \id[\ptstar \geq 0.5]\ptstar$, $\eta = \frac{\ln(2)((1-\gamma)\exp(- 2 D X)\frac{1}{K} + \gamma)}{2 K^2 X^2}$, and let $\ell_t$ be the bandit logistic loss \eqref{eq:bandit logistic loss}. Then there exists a setting of $\gamma$ such that Algorithm \ref{alg:gapjumper} satisfies
\begin{align*}
    \E\left[\sumT  \id[y_t' \not = y_t]\right] \leq & \Eb{\sumT \ell_t(\U)} +  K X D\min\left\{\max\left\{\frac{2K X D}{\ln(2)}, 2  \sqrt{\frac{T}{\ln(2)}}\right\}, \frac{K X D}{e^{- 2 D X}\ln(2)} \right\}.
\end{align*}
\end{theorem}

\subsection{Bandit Multiclass Hinge Loss}\label{sec: bandit hinge}

We use the following definition of the bandit multiclass hinge loss: \begin{equation}\label{eq:bandit multiclass hinge}
    \ell_t(\W_t) = 
    \begin{cases}
        \frac{\id[y_t' = y_t]}{p_t'(y_t')} \max\{1-m_t(\W_t, y_t), 0\} & \text{ if } \mstart \leq \beta  \\
        \frac{\id[y_t' = y_t]}{p_t'(y_t')} \max\{1-m_t(\W_t, y_t), 0\} & \text{ if } \ystart \not = y_t \text{ and } \mstart > \beta  \\
        0 & \text{ if } y_t' = \ystart = y_t \text{ and } \mstart > \beta.
        \end{cases}
\end{equation}

It is straightforward to see that the conditional expectation of the bandit multiclass hinge loss is the full information multiclass hinge loss. Both the \textproc{Banditron} algorithm \citep{kakade2008efficient} and \textproc{Soba} \citep{beygelzimer2017efficient} use a similar loss function. 

As we mentioned before, \citet{beygelzimer2017efficient} present \textproc{Soba}, which is a second-order algorithm with regret  $O(\|\U\|^2X^2 + \frac{K}{\kappa}\sqrt{dT\ln(T+1)})$. \textproc{Banditron} is a first-order algorithm based on the \textproc{Perceptron} algorithm and suffers $O((KDX)^{1/3}T^{2/3})$ regret. For the more general setting of contextual bandits \citep{foster2018contextual} use continuous Exponential Weights with the hinge loss to also obtain an $O(KDX\sqrt{dT\ln(T+1)})$ regret bound with a polynomial time algorithm. 
The expected mistake bound of  \textproc{Gaptron} can be found in Theorem \ref{th: bandit hinge} and its proof can be found in Appendix \ref{app:bandit hinge}. Compared to the \textproc{Banditron} \textproc{Gaptron} has larger regret in terms of $D$, $K$, and $X$, but smaller regret in terms of $T$. Compared to the regret of \textproc{Soba} the regret of \textproc{Gaptron} does not contain a factor $\sqrt{d\ln(T+1)}$.

\begin{theorem}\label{th: bandit hinge}
Set $a(\W_t, \x_t) = 1 - \max\{\id[\mstart > \beta], \mstart\}$, $\eta = \frac{\gamma(1-\beta)}{K^2 X^2}$, $\gamma = \min\left\{1, \sqrt{\frac{K^3X^2D^2}{2(1-\beta)(K-1)T}}\right\}$, and let $\ell_t$ be the bandit multiclass hinge loss defined in \eqref{eq:bandit multiclass hinge} with $\beta = \frac{1}{K}$. Then for any $\U \in \doma{W}$ Algorithm \ref{alg:gapjumper} satisfies
\begin{align*}
    \E\left[\sumT  \id[y_t' \not = y_t]\right] \leq &  \Eb{\sumT \ell_t(\U)} + \max\left\{\frac{K^3X^2D^2}{K-1}, ~ 2K X D \sqrt{\frac{T}{2}}\right\}.
\end{align*}
\end{theorem}

\subsection{Bandit Smooth Multiclass Hinge Loss}\label{sec: bandit smooth hinge}
In this section we use the following loss function:
\begin{equation}\label{eq:bandit multiclass smooth hinge}
    \ell_t(\W) = 
    \begin{cases}
        \frac{\id[y_t' = y_t]}{p_t'(y_t)} \max\{1 - 2 m_t(\W, y_t), 0\} & \text{ if }  m_t(\W, y_t) \leq 0  \\
        \frac{\id[y_t' = y_t]}{p_t'(y_t)} \max\{(1-m_t(\W, y_t))^2, 0\} & \text{ if }  m_t(\W, y_t) > 0.
        \end{cases}
\end{equation}
This loss function is the bandit version of the smooth multiclass hinge loss that we we used in Section \ref{sec: full smooth hinge} and its expectation is equivalent to its full information counterpart in equation \eqref{eq:multiclass smooth hinge}. The regret of \textproc{Gaptron} with this loss function can be found in Theorem \ref{th: bandit smooth hinge}. The proof of Theorem \ref{th: bandit smooth hinge} can be found in Appendix \ref{app:bandit smooth hinge}. 

\begin{theorem}\label{th: bandit smooth hinge}
Set $a(\W_t, \x_t) = (1 - \min\{1, \mstart\})^2$, $\eta = \frac{\gamma}{4 K^2 X^2}$, $\gamma = \min\left\{1, \sqrt{\frac{4K^2X^2D^2}{T}}\right\}$, and let $\ell_t$ be the bandit smooth multiclass hinge loss defined in \eqref{eq:bandit multiclass smooth hinge}. Then for any $\U \in \doma{W}$ Algorithm \ref{alg:gapjumper} satisfies
\begin{align*}
    \E\left[\sumT  \id[y_t' \not = y_t]\right] \leq & \Eb{\sumT \ell_t(\U)} + \max\left\{4K^2X^2D^2, ~ 2 KXD\sqrt{2T} \right\}.
\end{align*}
\end{theorem}

\section{Conclusion}
In this paper we introduced \textproc{Gaptron}, a randomized first-order algorithm for the full and bandit information multiclass classification settings. Using a new technique we showed that \textproc{Gaptron} has an $O(K)$ regret bound in the full information setting and a regret bound of $O(K\sqrt{T})$ in the bandit setting. 
One of the main drawbacks of \textproc{Gaptron} is that it is a randomized algorithm. Our bounds only hold in expectation and it would be interesting to show similar bounds also hold with high probability. Another interesting venue to explore is how to extend the ideas in this paper to the stochastic setting or the more general contextual bandit setting. In future work we would like to conduct experiments to compare \textproc{Gaptron} with other algorithms, particularly in the bandit setting. Finally, as the results of \citet{beygelzimer2019bandit} show, in the separable bandit setting \textproc{Gaptron} does not obtain the optimal regret bound. Deriving an algorithm  that both has $O(K\sqrt{T})$ regret in the adversarial bandit setting and $O(K)$ regret in the separable bandit setting is also an interesting direction to pursue. 

\paragraph{Acknowledgements}
The author would like to thank Tim van Erven and Sarah Sachs for their comments on an earlier version of this paper. The author would like to thank Francesco Orabona for pointing out two related references: \citet{crammer2013multiclass} and \citet{beygelzimer2019bandit}. The author was supported by the Netherlands Organization for Scientific Research (NWO grant
TOP2EW.15.211).

\bibliography{myBib}

\clearpage
\appendix

\section{Table with Different Settings of \textproc{Gaptron}}
\label{app:settings}

\begin{table}[h]
\caption{Settings of Gaptron}
\label{tab:settings}
\centering
\resizebox{\textwidth}{!}{\begin{tabular}{lcccc}
  \hline
Surrogate Loss & Gap map $a$ & Learning rate $\eta$ & Exploration $\gamma$ & Regret \\ 
  \hline
logistic \eqref{eq:full logistic loss} & $1 - \id[\ptstar \geq 0.5]\ptstar$ & $\frac{\ln(2)}{2 K X^2}$ & 0 & $\frac{K X^2 \|\U\|^2}{\ln(2)}$\\
bandit logistic \eqref{eq:bandit logistic loss} & $1 - \id[\ptstar \geq 0.5]\ptstar$ & $\frac{\ln(2) \exp(-2DX)}{2 K^2 X^2}$ & 0 & $\frac{\exp(2 D X)K^2 X^2 D^2}{\ln(2)}$ \\
bandit logistic \eqref{eq:bandit logistic loss} & $1 - \id[\ptstar \geq 0.5]\ptstar$ & $\frac{\gamma \ln(2)}{2 K^2 X^2 }$ & $\sqrt{\frac{2K^2X^2}{T}}$ & $2KXD \sqrt{\frac{T}{\ln(2)}}$ \\
hinge \eqref{eq:multiclass hinge} & $1 - \max\{\id[\mstart > \beta]\}$ & $\frac{K-1}{K^2 X^2}$ & 0 & $\frac{K^2X^2\|\U\|^2}{2(K-1)}$ \\
bandit hinge \eqref{eq:bandit multiclass hinge} & $1 - \max\{\id[\mstart > \beta]\}$ & $\frac{\gamma(K-1)}{K^3 X^2}$ & $\sqrt{\frac{K^4X^2D^2}{2(K-1)^2T}}$ & $2K X D \sqrt{\frac{T}{2}}$ \\
smooth hinge \eqref{eq:multiclass smooth hinge} & $(1 - \min\{1, \mstart\})^2$ & $\frac{1}{2 K X^2}$ & 0 & $2 K X^2 \|\U\|^2$ \\
bandit smooth hinge \eqref{eq:bandit multiclass smooth hinge} & $(1 - \min\{1, \mstart\})^2$ & $\frac{\gamma}{2 K^2 X^2}$ & $\sqrt{\frac{4K^2X^2D^2}{T}}$ & $2 DKX\sqrt{2T}$ \\
   \hline
\end{tabular}}
\end{table}

\section{Details of Section \ref{sec:gap}}\label{app:gap}

\begin{proof}[Proof of Lemma \ref{lem: surrogate gap}]
As we said before, the updates of $\W_t$ are Online Gradient Descent \citep{zinkevich2003}, which guarantees 
\begin{equation}\label{eq:regret gradient descent}
    \sumT \ell_t(\W_t) - \ell_t(\U) \leq \frac{\|\U\|^2}{2\eta} + \sumT \frac{\eta}{2}\|\g_t\|^2.
\end{equation}
Now, by using \eqref{eq:regret gradient descent} we find
\begin{equation}\label{eq:surrogate gap proof}
\begin{split}
    & \E\Bigg[\sumT  \left(\id[y_t' \not = y_t] - \ell_t(\U)\right)\Bigg] \\ 
    & =  \E\left[\sumT \left(\id[y_t' \not = y_t]] - \ell_t(\W_t)\right) + \sumT \left(\ell_t(\W_t) - \ell_t(\U)\right)\right] \\
    & \leq \frac{\|\U\|^2}{2\eta} + \E\left[\sumT \left( \E_t[\id[y_t' \not = y_t]] - \ell_t(\W_t) + \frac{\eta}{2}\|\g_t\|^2\right)\right] \\
    & =  \frac{\|\U\|^2}{2\eta} + \E\left[\sumT \left((1 - \max\{a_t, \gamma\})\id[y_t^\star \not = y_t] + \max\{a_t, \gamma\}\frac{K-1}{K} - \ell_t(\W_t) + \frac{\eta}{2}\|\g_t\|^2\right)\right] \\
    & \leq \frac{\|\U\|^2}{2\eta} + \gamma \frac{K-1}{K} T + \E\left[\sumT \left((1 - a_t)\id[y_t^\star \not = y_t] + a_t \frac{K-1}{K} - \ell_t(\W_t) + \frac{\eta}{2}\|\g_t\|^2\right)\right],
\end{split}
\end{equation}
where in the last inequality we used $(1 - \max\{a_t, \gamma\}) \leq (1 - a_t)$ and $\max\{a_t, \gamma\} \leq a_t + \gamma$. Adding $\E\left[\sumT \ell_t(\U)]\right]$ to both sides of equation \eqref{eq:surrogate gap proof} completes the proof.
\end{proof}

\section{Details of Full Information Multiclass Classification (Section \ref{sec:full info})}\label{app:full info}

\subsection{Details of Section \ref{sec: full logistic}}\label{app:full logistic}

\begin{proof}[Proof of Theorem \ref{th:full logistic}]
We will prove the Theorem by showing that the surrogate gap is bounded by 0 and then using Lemma \ref{lem: surrogate gap}. The gradient of the logistic loss evaluated at $\W_t$ is given by:
\begin{align*}
    \nabla\ell_t(\W_t) = & \frac{1}{\ln(2)}(\tilde{\p}_t - \e_{y_t})\otimes\x_t, 
\end{align*}
where $\tilde{\p}_t = (\tilde{p}_t(1), \ldots, \tilde{p}_t(k))^\top$ and $\tilde{p}_t(k) = \sigma(\W_t, \x_t, k)$.

 We continue by writing out the surrogate gap:
\begin{equation}\label{eq:cases full log}
    \begin{split}
    & (1 - a_t)\id[\ystart \not = y_t] + a_t \frac{K-1}{K} - \ell_t(\W_t) + \frac{\eta}{2} \|\g_t\|^2 \\
    & \leq  (1 - a_t)\id[\ystart \not = y_t] + a_t \frac{K-1}{K} - \ell_t(\W_t) - \frac{\eta}{\ln(2)} \|\x_t\|^2 \log_2(\tilde{p}_t(y_t))  \\
    & \leq  (1 - a_t)\id[\ystart \not = y_t] + a_t \frac{K-1}{K} - \ell_t(\W_t) - \frac{\eta}{\ln(2)} X^2 \log_2(\tilde{p}_t(y_t))  \\
    & = \begin{cases}
    0 + \frac{K-1}{K} + \log_2(\tilde{p}_t(y_t)) - \frac{\eta}{\ln(2)}X^2 \log_2(\tilde{p}_t(y_t)) & \text{ if } \ptstar < 0.5 \\
    \ptstar + (1 - \ptstar) \frac{K-1}{K} + \log_2(\tilde{p}_t(y_t)) - \frac{\eta}{\ln(2)}X^2 \log_2(\tilde{p}_t(y_t)) & \text{ if } y_t^\star \not = y_t \text{ and } \ptstar \geq 0.5 \\
    (1 - \ptstar) \frac{K-1}{K} + \log_2(\ptstar) - \frac{\eta}{\ln(2)}X^2 \log_2(\ptstar) & \text{ if } y_t^\star = y_t \text{ and } \ptstar \geq 0.5,
    \end{cases}
    \end{split}
\end{equation}
where the first inequality is due to Lemma \ref{lem: pinsker log loss} below.

We now split the analysis into the cases in \eqref{eq:cases full log}. We start with $\ptstar < 0.5$. In this case we use $1 \leq -\log_2(x) $ for $x \in [0, \half]$ and obtain 
\begin{align*}
    & \frac{K-1}{K}  + \log_2(\tilde{p}_t(y_t)) - \frac{\eta}{\ln(2)}X^2 \log_2(\tilde{p}_t(y_t)) \\
    & \leq  -\frac{K-1}{K}\log_2(\tilde{p}_t(y_t)) + \log_2(\tilde{p}_t(y_t)) - \frac{\eta}{\ln(2)} X^2 \log_2(\tilde{p}_t(y_t))\\
    & = \frac{1}{K}\log_2(\tilde{p}_t(y_t)) - \frac{\eta}{\ln(2)} X^2 \log_2(\tilde{p}_t(y_t)),
\end{align*}
which is bounded by 0 since $\eta < \frac{\ln(2)}{K X^2}$.

The second case we consider is when $y_t^\star \not = y_t \text{ and } \ptstar \geq 0.5$. In this case we use $x \leq -\frac{1}{2} \log_2(1 - x)$ for $x \in [0.5, 1]$ and $1-x \leq -\frac{1}{2} \log_2(1 - x)$ for $x \in [0.5, 1]$ and obtain
\begin{align*}
    & \ptstar + (1 - \ptstar) \frac{K-1}{K} + \log_2(\tilde{p}_t(y_t)) - \frac{\eta}{\ln(2)}X^2 \log_2(\tilde{p}_t(y_t))\\
    & \leq -\half \log_2(1 - \ptstar) - \frac{K-1}{K}\half \log_2(1 - \ptstar) + \log_2(\tilde{p}_t(y_t)) - \frac{\eta}{\ln(2)}X^2 \log_2(\tilde{p}_t(y_t))\\
    & = -\half \log_2\left(\sum_{k \not =  y_t}^K \tilde{p}_t(k) \right) - \frac{K-1}{K}\half \log_2\left(\sum_{k \not =  y_t}^K \tilde{p}_t(k)  \right) + \log_2(\tilde{p}_t(y_t)) - \frac{\eta}{\ln(2)}X^2 \log_2(\tilde{p}_t(y_t)) \\
    & \leq -\half \log_2\left(\tilde{p}_t(y_t)\right) - \frac{K-1}{K}\half \log_2\left( \tilde{p}_t(y_t) \right) + \log_2(\tilde{p}_t(y_t)) - \frac{\eta}{\ln(2)}X^2 \log_2(\tilde{p}_t(y_t)) \\
    & =  \frac{1}{2 K} \log_2\left( \tilde{p}_t(y_t) \right) - \frac{\eta}{\ln(2)}X^2 \log_2(\tilde{p}_t(y_t)),
\end{align*}
which is 0 since $\eta = \frac{\ln(2)}{2 K X^2}$.

The last case we need to consider is $y_t^\star = y_t \text{ and } \ptstar \geq 0.5$. In this case we use $1 - x \leq -\log_2(x)$ and obtain 
\begin{align*}
    & (1 - \ptstar) \frac{K-1}{K} + \log_2(\ptstar) - \frac{\eta}{\ln(2)}X^2 \log_2(\ptstar) \\
    & \leq -\frac{K-1}{K}\log_2(\ptstar) + \log_2(\ptstar) - \frac{\eta}{\ln(2)}X^2 \log_2(\ptstar),
\end{align*}
which is bounded by 0 since $\eta = \frac{\ln(2)}{2 K X^2}$. 

We now apply Lemma \ref{lem: surrogate gap}, plug in $\gamma = 0$, and use the above to find:
\begin{align*}
    \E\left[\sumT  \id[y_t' \not = y_t]\right] \leq & \frac{\|\U\|^2}{2\eta} + \sumT \ell_t(\U) + \gamma \frac{K-1}{K} T  \\
     & + \sumT \left((1 - a_t)\id[\ystart \not = y_t] + a_t \frac{K-1}{K} - \ell_t(\W_t) + \frac{\eta}{2}\|\g_t\|^2\right) \\
     \leq & \frac{\|\U\|^2}{2\eta} + \sumT \ell_t(\U).
\end{align*}
Using $\eta = \frac{\ln(2)}{2 K X^2}$ completes the proof.

\end{proof}

\begin{lemma}\label{lem: pinsker log loss}
Let $\ell_t$ be the logistic loss \eqref{eq:full logistic loss}, then 
\begin{align*}
    \|\nabla \ell_t(\W_t)\|^2 \leq & \frac{2}{\ln(2)} \|\x_t\|^2 \ell_t(\W_t).
\end{align*}
\end{lemma}
\begin{proof}
We have
\begin{equation*}
\begin{split}
    \|\nabla \ell_t(\W_t)\|^2 = & \frac{1}{\ln(2)^2}\|\x_t\|^2\left(\sumk (\id[y_t = k] - \tilde{p}_t(k))^2\right) \\
    \leq & \frac{1}{\ln(2)^2}\|\x_t\|^2\left(\sumk |\id[y_t = k] - \tilde{p}_t(k)|\right)^2 \\
    \leq & - 2\frac{1}{\ln(2)} \|\x_t\|^2 \log_2(\tilde{p}_t(y_t)) \\
    = & 2\frac{1}{\ln(2)} \|\x_t\|^2 \ell_t(\W_t), 
\end{split}
\end{equation*}
where the last inquality follows from Pinsker's inequality \citep[Lemma 12.6.1]{cover1991elements}.
\end{proof}

\subsection{Details of Section \ref{sec: full hinge}}\label{app:full hinge}

\begin{proof}[Proof of Theorem \ref{th: full hinge}]
We will prove the Theorem by showing that the surrogate gap is bounded by 0 and then using Lemma \ref{lem: surrogate gap}. Let $\tilde{k} = \argmax_{k \not = y_t} \inner{\W_t^k}{\x_t}$. The gradient of the smooth multiclass hinge loss is given by 
\begin{align*}
    \nabla \ell_t(\W_t) = 
    \begin{cases}
    (\e_{\tilde{k}} - \e_{y_t}) \otimes \x_t & \text{ if } \ystart \not = y_t \\
    (\e_{\tilde{k}} - \e_{y_t}) \otimes \x_t & \text{ if } \ystart = y_t \text{ and } \mstart \leq \beta\\
    0 & \text{ if } \ystart = y_t \text{ and } \mstart > \beta.
    \end{cases}
\end{align*}

We continue by writing out the surrogate gap:
\begin{equation}\label{eq:cases full hinge}
    \begin{split}
    & (1 - a_t)\id[\ystart \not = y_t] + a_t \frac{K-1}{K} - \ell_t(\W_t) + \frac{\eta}{2} \|\g_t\|^2 \\
    & = \begin{cases}
    \mstart + (1 - \mstart)\frac{K-1}{K} - (1-m_t(\W_t, y_t)) + \eta \|\x_t\|^2 & \text{ if } \ystart \not = y_t \text{ and } \mstart \leq \beta \\
    (1 - \mstart)\frac{K-1}{K} - (1-\mstart) + \eta \|\x_t\|^2 & \text{ if } \ystart = y_t \text{ and } \mstart \leq \beta \\
    1 - (1-m_t(\W_t, y_t)) + \eta \|\x_t\|^2 & \text{ if } \ystart \not = y_t \text{ and } \mstart > \beta \\
    0 & \text{ if } \ystart = y_t \text{ and } \mstart > \beta.
    \end{cases}
    \end{split}
\end{equation}
In the remainder of the proof we will repeatedly use the following useful inequality for whenever $y_t \not= \ystart$:
\begin{equation}\label{eq:mtstar + mt}
\begin{split}
    \mstart + m_t(\W_t, y_t) = & \inner{\W_t^{\ystart}}{\x_t} - \max_{k \not= \ystart}\inner{\W_t^k}{\x_t} +  \inner{\W_t^{y_t}}{\x_t} - \max_{k \not= y_t} \inner{\W_t^k}{\x_t} \\
    = & \inner{\W_t^{y_t}}{\x_t} - \max_{k \not= \ystart}\inner{\W_t^k}{\x_t}\\
    \leq & \inner{\W_t^{y_t}}{\x_t} - \inner{\W_t^{y_t}}{\x_t} = 0.
\end{split}
\end{equation}
We now split the analysis into the cases in \eqref{eq:cases full hinge}. We start with $\ystart \not = y_t \text{ and } \mstart \leq \beta$, in which case the surrogate gap can be bounded by 0 when $\eta \leq \frac{1}{K X^2}$:
\begin{align*}
    &\mstart + (1 - \mstart)\frac{K-1}{K} - (1-m_t(\W_t, y_t)) + \eta \|\x_t\|^2 \\
    & =  \mstart + m_t(\W_t, y_t) + (1 - \mstart)\frac{K-1}{K} - 1+ \eta \|\x_t\|^2 \\
    & \leq -\frac{1}{K} + \eta X^2 \tag{by equation \eqref{eq:mtstar + mt}}\\
    & \leq 0.
\end{align*}

We continue with the case where $\ystart = y_t \text{ and } \mstart \leq \beta$. In this case we have:
\begin{align*}
    (1 - \mstart)\frac{K-1}{K} - (1-\mstart) + \eta \|\x_t\|^2 & = -(1 - \mstart)\frac{1}{K}+ \eta \|\x_t\|^2  \leq -\frac{1-\beta}{K} + \eta X^2,
\end{align*}
which is zero since $\eta = \frac{1-\beta}{K X^2}$.

Finally, in the case where $\ystart \not = y_t \text{ and } \mstart > \beta$ we have:
\begin{align*}
    1 - (1-m_t(\W_t, y_t)) + \eta \|\x_t\|^2 = & m_t(\W_t, y_t) + \eta \|\x_t\|^2 \\
    \leq & -\mstart + \eta \|\x_t\|^2 \tag{by equation \eqref{eq:mtstar + mt}}\\
    \leq & -\beta + \eta X^2,
\end{align*}
which is bounded by zero since $\beta =\frac{1}{K}$ and $\eta \leq \frac{1}{K X^2}$.

We now apply Lemma \ref{lem: surrogate gap}, plug in $\gamma = 0$, and use the above to find:
\begin{align*}
    \E\left[\sumT  \id[y_t' \not = y_t]\right] \leq & \frac{\|\U\|^2}{2\eta} + \sumT \ell_t(\U) + \gamma T  \\
     & + \sumT \left((1 - a_t)\id[\ystart \not = y_t] + a_t \frac{K-1}{K} - \ell_t(\W_t) + \frac{\eta}{2}\|\g_t\|^2\right) \\
     \leq & \frac{\|\U\|^2}{2\eta} + \sumT \ell_t(\U).
\end{align*}
Using $\eta = \frac{1-\beta}{K X^2} = \frac{K-1}{K^2X^2}$ completes the proof.
\end{proof}

\subsection{Details of Section \ref{sec: full smooth hinge}}\label{app:full smooth hinge}

\begin{proof}[Proof of Theorem \ref{th: full smooth hinge}]
We will prove the Theorem by showing that the surrogate gap is bounded by 0 and then using Lemma \ref{lem: surrogate gap}. Let $\tilde{k} = \argmax_{k \not= y_t} \inner{\W_t^k}{\x_t}$. The gradient of the smooth multiclass hinge loss is given by 
\begin{align*}
    \nabla \ell_t(\W_t) = 
    \begin{cases}
    2(\e_{\tilde{k}} - \e_{y_t}) \otimes \x_t & \text{ if } \ystart \not = y_t \\
    2(\e_{\tilde{k}} - \e_{y_t})(1 - \mstart) \otimes \x_t & \text{ if } \ystart = y_t \text{ and } \mstart < 1\\
    0 & \text{ if } \ystart = y_t \text{ and } \mstart \geq 1.
    \end{cases}
\end{align*}

We continue by writing out the surrogate gap:
\begin{equation}\label{eq:cases full smooth hinge}
    \begin{split}
    & (1 - a_t)\id[\ystart \not = y_t] + a_t \frac{K-1}{K} - \ell_t(\W_t) + \frac{\eta}{2} \|\g_t\|^2 \\
    & = \begin{cases}
    2\mstart - {\mstart}^2 + (1 -  \mstart)^2\frac{K-1}{K} - (1 - 2m_t(\W_t, y_t)) + \eta 4 \|\x_t\|^2 & \text{ if } \ystart \not = y_t \text{ and } \mstart < 1\\
    (1 -  \mstart)^2\frac{K-1}{K} - (1 -  \mstart)^2 + \eta 4 \|\x_t\|^2(1 -  \mstart)^2 & \text{ if } \ystart = y_t \text{ and } \mstart < 1\\
    1 - (1 - 2 m_t(\W_t, y_t)) + \eta 4 \|\x_t\|^2 & \text{ if } \ystart \not = y_t \text{ and } \mstart \geq 1 \\
    0 & \text{ if } \ystart = y_t \text{ and } \mstart \geq 1.
    \end{cases}
    \end{split}
\end{equation}
We now split the analysis into the cases in \eqref{eq:cases full smooth hinge}. We start with the case where $\ystart \not = y_t \text{ and } \mstart < 1$. By using \eqref{eq:mtstar + mt} we can see that with $\eta = \frac{1}{4 K X^2}$ the surrogate gap is bounded by 0:
\begin{align*}
    & 2\mstart - {\mstart}^2 + (1 -  \mstart)^2\frac{K-1}{K} - (1 - 2m_t(\W_t, y_t)) + \eta 4 \|\x_t\|^2 \\
    & = 2(\mstart + m_t(\W_t, y_t)) - {\mstart}^2 + (1 -  \mstart)^2\frac{K-1}{K} - 1 + \eta 4 \|\x_t\|^2 \\
    & \leq - {\mstart}^2 + (1 -  \mstart)^2\frac{K-1}{K} - 1 + \eta 4 X^2 \tag{by equation \eqref{eq:mtstar + mt}}\\
    & \leq -\frac{1}{K} + \eta 4 X^2 \leq 0.
\end{align*}

The next case we consider is when $\ystart = y_t \text{ and } \mstart < 1$. In this case we have 
\begin{align*}
     (1 -  \mstart)^2\frac{K-1}{K} - (1 -  \mstart)^2 + \eta 4 \|\x_t\|^2(1 -  \mstart)^2 
    & = -(1 -  \mstart)^2\frac{1}{K} + \eta 4 \|\x_t\|^2(1 -  \mstart)^2,
\end{align*}
which is bounded by 0 since $\eta = \frac{1}{4 K X^2}$.

Finally, if $\ystart \not = y_t \text{ and } \mstart \geq 1$ then 
\begin{align*}
    1 - (1 - 2 m_t(\W_t, y_t)) + \eta 4 \|\x_t\|^2 = & 2 m_t(\W_t, y_t) + \eta 4 \|\x_t\|^2 \\
    \leq & -2\mstart + \eta 4\|\x_t\|^2 \tag{by equation \eqref{eq:mtstar + mt}}\\
    \leq & -2 + \eta 4X^2,
\end{align*}
which is bounded by 0 since $\eta < \frac{1}{2X^2}$. We apply Lemma \ref{lem: surrogate gap} with $\gamma = 0$ and use the above to find:
\begin{align*}
    \E\left[\sumT  \id[y_t' \not = y_t]\right] \leq & \frac{\|\U\|^2}{2\eta} + \sumT \ell_t(\U) + \gamma \frac{K-1}{K} T  \\
     & + \sumT \left((1 - a_t)\id[\ystart \not = y_t] + a_t \frac{K-1}{K} - \ell_t(\W_t) + \frac{\eta}{2}\|\g_t\|^2\right) \\
     \leq & \frac{\|\U\|^2}{2\eta} + \sumT \ell_t(\U).
\end{align*}
Using $\eta = \frac{1}{4 K X^2}$ completes the proof.

\end{proof}

\section{Details of Bandit Multiclass Classification (Section \ref{sec: bandit})}\label{app:bandit}

\subsection{Details of Section \ref{sec: bandit logistic}}\label{app:bandit logistic}

\begin{proof}[Proof of Theorem \ref{th:bandit logistic}]
First, by straightforward calculations we can see that $p_t'(y_t) \geq \frac{(1-\gamma)\exp(- 2 D X) + \gamma}{K} = \delta$. As in the full information case we will prove the Theorem by showing that the surrogate gap is bounded by 0 and then using Lemma \ref{lem: surrogate gap}. We start by writing out the surrogate gap:
\begin{equation}\label{eq:cases bandit log}
    \begin{split}
    &  \E\left[(1 - a_t)\id[\ystart \not = y_t] + a_t \frac{K-1}{K} -\E_t[\ell_t(\W_t)] + \frac{\eta}{2} \E_t\left[\|\g_t\|^2\right]\right] \\
    & = \E\left[(1 - a_t)\id[\ystart \not = y_t] + a_t \frac{K-1}{K} + \log_2(\tilde{p}_t(y_t)) + \frac{\eta}{2\ln(2)^2p_t'(y_t)} \|(\tilde{\p}_t - \e_{y_t})\otimes\x_t\|^2\right] \\
    & \leq  \E\left[(1 - a_t)\id[\ystart \not = y_t] + a_t \frac{K-1}{K} + \log_2(\tilde{p}_t(y_t)) - \frac{\eta}{\ln(2)p_t'(y_t)} X^2 \log_2(\tilde{p}_t(y_t))\right]  \\
    & = \begin{cases}
    \frac{K-1}{K} + \E\left[\log_2(\tilde{p}_t(y_t)) - \frac{\eta}{\ln(2)p_t'(y_t)}X^2 \log_2(\tilde{p}_t(y_t))\right] & \text{ if } \ptstar < 0.5 \\
    \E\left[\ptstar + (1 - \ptstar) \frac{K-1}{K} + \log_2(\tilde{p}_t(y_t)) - \frac{\eta}{\ln(2)p_t'(y_t)}X^2 \log_2(\tilde{p}_t(y_t))\right] & \text{ if } y_t^\star \not = y_t \text{ and } \ptstar \geq 0.5 \\
    \E\left[(1 - \ptstar) \frac{K-1}{K} + \log_2(\ptstar) - \frac{\eta}{\ln(2)p_t'(\ystart)}X^2 \log_2(\ptstar)\right] & \text{ if } y_t^\star = y_t \text{ and } \ptstar \geq 0.5,
    \end{cases}
    \end{split}
\end{equation}
where the first inequality is due to Lemma \ref{lem: pinsker log loss}.

We now split the analysis into the cases in \eqref{eq:cases bandit log}. We start with $\ptstar < 0.5$. In this case we use $1 \leq -\log_2(x) $ for $x \in [0, \half]$ and obtain 
\begin{align*}
    & \frac{K-1}{K}  + \E[\log_2(\tilde{p}_t(y_t)) - \frac{\eta}{\ln(2)p_t'(y_t)}X^2 \log_2(\tilde{p}_t(y_t))] \\
    & \leq  \E\left[-\frac{K-1}{K}\log_2(\tilde{p}_t(y_t)) + \log_2(\tilde{p}_t(y_t)) - \frac{\eta}{\ln(2)p_t'(y_t)} X^2 \log_2(\tilde{p}_t(y_t))\right]\\
    & \leq \E\left[-\frac{K-1}{K}\log_2(\tilde{p}_t(y_t)) + \log_2(\tilde{p}_t(y_t)) - \frac{\eta}{\ln(2)\delta} X^2 \log_2(\tilde{p}_t(y_t))\right]\\
\end{align*}
which is bounded by 0 when $\eta \leq \frac{\ln(2)\delta}{K X^2}$.

The second case we consider is when $y_t^\star \not = y_t \text{ and } \ptstar \geq 0.5$. In this case we use $x \leq -\frac{1}{2} \log_2(1 - x)$ for $x \in [0.5, 1]$ and $1-x \leq -\frac{1}{2} \log_2(1 - x)$ for $x \in [0.5, 1]$ and obtain
\begin{align*}
    & \E\left[\ptstar + (1 - \ptstar) \frac{K-1}{K} + \log_2(\tilde{p}_t(y_t)) - \frac{\eta}{\ln(2)p_t'(y_t)}X^2 \log_2(\tilde{p}_t(y_t))\right]\\
    & \leq \E\left[-\half \log_2(1 - \ptstar) - \frac{K-1}{K}\half \log_2(1 - \ptstar) + \log_2(\tilde{p}_t(y_t)) - \frac{\eta}{\ln(2)\delta}X^2 \log_2(\tilde{p}_t(y_t))\right]\\
    & =\E\left[-\half \log_2\left(\sum_{k \not =  y_t}^K \tilde{p}_t(k) \right) - \frac{K-1}{K}\half \log_2\left(\sum_{k \not =  y_t}^K \tilde{p}_t(k)  \right) + \log_2(\tilde{p}_t(y_t)) - \frac{\eta}{\ln(2)\delta}X^2 \log_2(\tilde{p}_t(y_t))\right]\\
    & \leq \E\left[-\half \log_2\left(\tilde{p}_t(y_t)\right) - \frac{K-1}{K}\half \log_2\left( \tilde{p}_t(y_t) \right) + \log_2(\tilde{p}_t(y_t)) - \frac{\eta}{\ln(2)\delta}X^2 \log_2(\tilde{p}_t(y_t))\right] \\
    & = \E\left[\frac{1}{2K} \log_2(\tilde{p}_t(y_t)) - \frac{\eta}{\ln(2)\delta}X^2 \log_2(\tilde{p}_t(y_t))\right],
\end{align*}
which is bounded by 0 since $\eta = \frac{\ln(2)\delta}{2 K X^2}$.

The last case we need to consider is when $ y_t^\star = y_t \text{ and } \ptstar \geq 0.5$. In this case we use $1 - x \leq -\log_2(x)$ and obtain 
\begin{align*}
    & \E\left[(1 - \ptstar) \frac{K-1}{K} + \log_2(\ptstar) - \frac{\eta}{\ln(2)p_t'(\ystart)}X^2 \log_2(\ptstar)\right] \\
    & \leq \E\left[-\frac{K-1}{K}\log_2(\ptstar) + \log_2(\ptstar) - \frac{\eta}{\ln(2)\delta}X^2 \log_2(\ptstar)\right],
\end{align*}
which is bounded by 0 when $\eta \leq \frac{\ln(2)\delta}{K X^2}$. 

We now apply Lemma \ref{lem: surrogate gap} and use the above to find:
\begin{align*}
    \E\left[\sumT  \id[y_t' \not = y_t]\right] \leq & \frac{\|\U\|^2}{2\eta} + \E\left[\sumT \ell_t(\U)\right] + \gamma \frac{K-1}{K}T  \\
     & + \E\left[\sumT \left((1 - a_t)\id[\ystart \not = y_t] + a_t \frac{K-1}{K} - \ell_t(\W_t) + \frac{\eta}{2}\|\g_t\|^2\right)\right] \\
     \leq & \frac{\|\U\|^2}{2\eta} + \gamma T + \E\left[\sumT \ell_t(\U)\right].
\end{align*}
Using $\eta = \frac{\ln(2)\delta}{2 K X^2}$ gives us:
\begin{align*}
    \E\left[\sumT  \id[y_t' \not = y_t]\right] \leq & \frac{K^2 X^2 \|\U\|^2}{\ln(2)((1-\gamma)\exp(- 2 D X) + \gamma)} + \gamma T + \E\left[\sumT \ell_t(\U)\right],
\end{align*}
Setting $\gamma = 0$ gives us
\begin{align*}
    \E\left[\sumT  \id[y_t' \not = y_t]\right] \leq & \frac{K^2 X^2 D^2}{\ln(2)\exp(- 2 D X)} + \E\left[\sumT \ell_t(\U)\right].
\end{align*}
If instead we set $\gamma = \min\left\{1, \sqrt{\frac{K^2X^2D^2}{\ln(2)T}}\right\}$ we consider two cases. In the case where $1 \leq \sqrt{\frac{K^2X^2D^2}{T}}$ we have that $T \leq K^2X^2D^2$ and therefore 
\begin{align*}
    \E\left[\sumT  \id[y_t' \not = y_t]\right] \leq & 2 \frac{K^2 X^2 D^2}{\ln(2)} + \E\left[\sumT \ell_t(\U)\right].
\end{align*}
In the case where $1 > \sqrt{\frac{K^2X^2D^2}{T}}$ we have that 
\begin{align*}
    \E\left[\sumT  \id[y_t' \not = y_t]\right] \leq & 2 K X D\sqrt{\frac{T}{\ln(2)}} + \E\left[\sumT \ell_t(\U)\right],
\end{align*}
which after combining the above completes the proof.
\end{proof}

\subsection{Details of Section \ref{sec: bandit hinge}}\label{app:bandit hinge}

\begin{proof}[Proof of Theorem \ref{th: bandit hinge}]
First, note that $p_t'(y_t) \geq \frac{\gamma}{K}$. The proof proceeds in a similar way as in the full information setting (Theorem \ref{th: full hinge}), except now we use that $p_t'(y_t) \geq \frac{\gamma}{K}$ to bound $\E_t[\|\g_t\|^2]$. We will prove the Theorem by showing that the surrogate gap is bounded by 0 and then using Lemma \ref{lem: surrogate gap}. We start by splitting the surrogate gap in cases:
\begin{equation}\label{eq:cases bandit hinge}
    \begin{split}
    & \Eb{(1 - a_t)\id[\ystart \not = y_t] + a_t \frac{K-1}{K} - \E_t[\ell_t(\W_t)] + \frac{\eta}{2} \E_t[\|\g_t\|^2]} \\
    & = \begin{cases}
    \Eb{\mstart + (1 - \mstart)\frac{K-1}{K} - (1-m_t(\W_t, y_t)) + \frac{\eta}{p_t'(y_t)} \|\x_t\|^2} & \text{ if } \ystart \not = y_t \text{ and } \mstart \leq \beta \\
    \Eb{(1 - \mstart)\frac{K-1}{K} - (1-\mstart) + \frac{\eta}{p_t'(y_t)} \|\x_t\|^2} & \text{ if } \ystart = y_t \text{ and } \mstart \leq \beta \\
    \Eb{1 - (1-m_t(\W_t, y_t)) + \frac{\eta}{p_t'(y_t)} \|\x_t\|^2} & \text{ if } \ystart \not = y_t \text{ and } \mstart > \beta \\
    0 & \text{ if } \ystart = y_t \text{ and } \mstart > \beta.
    \end{cases}
    \end{split}
\end{equation}
We now split the analysis into the cases in \eqref{eq:cases bandit hinge}. We start with $\ystart \not = y_t \text{ and } \mstart \leq \beta$. 
The surrogate gap can now be bounded by 0 when $\eta \leq \frac{\gamma}{K^2 X^2}$:
\begin{align*}
    &\Eb{\mstart + (1 - \mstart)\frac{K-1}{K} - (1-m_t(\W_t, y_t)) + \frac{\eta}{p_t'(y_t)} \|\x_t\|^2} \\
    & =  \Eb{\mstart + m_t(\W_t, y_t) + (1 - \mstart)\frac{K-1}{K} - 1+ \frac{\eta}{p_t'(y_t)} \|\x_t\|^2} \\
    & \leq -\frac{1}{K} + \frac{K\eta}{\gamma} X^2 \tag{equation \eqref{eq:mtstar + mt}}\\
    & \leq 0.
\end{align*}

We continue with the case where $\ystart = y_t \text{ and } \mstart \leq \beta$. In this case we have:
\begin{align*}
     \Eb{(1 - \mstart)\frac{K-1}{K} - (1-\mstart) + \eta \|\x_t\|^2}
    & = \Eb{-(1 - \mstart)\frac{1}{K}+ \frac{\eta}{p_t'(y_t)} \|\x_t\|^2} \\
    & \leq -\frac{1-\beta}{K} + \frac{K\eta}{\gamma} X^2,
\end{align*}
which is bounded by zero since $\eta = \frac{\gamma(1-\beta)}{K^2 X^2}$.

Finally, in the case where $\ystart \not = y_t \text{ and } \mstart > \beta$ we have:
\begin{align*}
    \Eb{1 - (1-m_t(\W_t, y_t)) + \frac{\eta}{p_t'(y_t)} \|\x_t\|^2} = & \Eb{m_t(\W_t, y_t) + \frac{\eta}{p_t'(y_t)} \|\x_t\|^2} \\
    \leq & \Eb{-\mstart + \frac{\eta}{p_t'(y_t)} \|\x_t\|^2} \tag{by equation \eqref{eq:mtstar + mt}}\\
    \leq & -\beta + \frac{K\eta}{\gamma} X^2,
\end{align*}
which is bounded by zero since $\eta = \frac{\gamma(1-\beta)}{K^2 X^2}$ and $\beta \leq 0.5$.

We now apply Lemma \ref{lem: surrogate gap} and use the above to find:
\begin{align*}
    \Eb{\sumT \id[y_t' \not = y_t]} \leq & \frac{\|\U\|^2}{2\eta} + \Eb{\sumT \ell_t(\U)} + \gamma \frac{K-1}{K} T  \\
     & + \Eb{\sumT \left((1 - a_t)\id[\ystart \not = y_t] + a_t \frac{K-1}{K} - \ell_t(\W_t) + \frac{\eta}{2}\|\g_t\|^2\right)} \\
     \leq & \frac{D^2}{2\eta} + \gamma \frac{K-1}{K} T + \Eb{\sumT \ell_t(\U)}.
\end{align*}
Plugging in  $\eta = \frac{\gamma(1-\beta)}{K^2 X^2}$ and $\beta = \frac{1}{K}$ gives us:
\begin{align*}
    \E\left[\sumT  \id[y_t' \not = y_t]\right] \leq & \frac{K^3X^2D^2}{2 \gamma(K-1)} + \gamma \frac{K-1}{K} T+ \Eb{\sumT \ell_t(\U)}.
\end{align*}
We now set $\gamma = \min\left\{1, \sqrt{\frac{K^3X^2D^2}{2(1-\beta)(K-1)T}}\right\}$. In the case where $1 \leq \sqrt{\frac{K^3X^2D^2}{2(1-\beta)(K-1)T}}$ we have
\begin{align*}
    \E\left[\sumT  \id[y_t' \not = y_t]\right] \leq & \frac{K^3X^2D^2}{K-1} + \Eb{\sumT \ell_t(\U)}.
\end{align*}
In the case where $1 > \sqrt{\frac{K^3X^2D^2}{2(1-\beta)(K-1)T}}$ we have
\begin{align*}
    \E\left[\sumT  \id[y_t' \not = y_t]\right] \leq & 2 K X D \sqrt{\frac{T}{2}} + \Eb{\sumT \ell_t(\U)},
\end{align*}
which completes the proof.
\end{proof}

\subsection{Details of Section \ref{sec: bandit smooth hinge}}\label{app:bandit smooth hinge}

\begin{proof}[Proof of Theorem \ref{th: bandit smooth hinge}]
First, note that $p_t'(y_t) \geq \frac{\gamma}{K}$. The proof proceeds in a similar way as in the full information case. We will prove the Theorem by showing that the surrogate gap is bounded by 0 and then using Lemma \ref{lem: surrogate gap}. We start by writing out the surrogate gap:
\begin{equation}\label{eq:cases bandit smooth hinge}
    \begin{split}
    & \Eb{(1 - a_t)\id[\ystart \not = y_t] + a_t \frac{K-1}{K} - \E_t[\ell_t(\W_t)] + \frac{\eta}{2} \E_t[\|\g_t\|^2]} \\
    & = \begin{cases}
    \Eb{2\mstart - {\mstart}^2 + (1 -  \mstart)^2\frac{K-1}{K} - (1 - 2m_t(\W_t, y_t)) + \frac{\eta}{p_t'(y_t)} 4 \|\x_t\|^2} & \text{ if } \ystart \not = y_t \text{ and } \mstart < 1\\
    \Eb{(1 -  \mstart)^2\frac{K-1}{K} - (1 -  \mstart)^2 + \frac{\eta}{p_t'(y_t)} 4 \|\x_t\|^2(1 -  \mstart)^2} & \text{ if } \ystart = y_t \text{ and } \mstart < 1\\
    \Eb{1 - (1 - 2 m_t(\W_t, y_t)) + \frac{\eta}{p_t'(y_t)} 4 \|\x_t\|^2} & \text{ if } \ystart \not = y_t \text{ and } \mstart \geq 1 \\
    0 & \text{ if } \ystart = y_t \text{ and } \mstart \geq 1.
    \end{cases}
    \end{split}
\end{equation}
We now split the analysis into the cases in \eqref{eq:cases bandit smooth hinge}. We start with the case where $\ystart \not = y_t \text{ and } \mstart < 1$. By using \eqref{eq:mtstar + mt} we can see that for $\eta = \frac{\gamma}{4 K^2 X^2}$
\begin{align*}
    & \Eb{2\mstart - {\mstart}^2 + (1 -  \mstart)^2\frac{K-1}{K} - (1 - 2m_t(\W_t, y_t)) + \frac{\eta}{p_t'(y_t)} 4 \|\x_t\|^2} \\
    & = \Eb{2(\mstart + m_t(\W_t, y_t)) - {\mstart}^2 + (1 -  \mstart)^2\frac{K-1}{K} - 1 + \frac{\eta}{p_t'(y_t)} 4 \|\x_t\|^2} \\
    & \leq \Eb{- {\mstart}^2 + (1 -  \mstart)^2\frac{K-1}{K} - 1 + \frac{\eta}{p_t'(y_t)} 4 X^2} \tag{by equation \eqref{eq:mtstar + mt}}\\
    & \leq -\frac{1}{K} + \frac{K \eta}{\gamma} 4 X^2 \leq 0.
\end{align*}

The next case we consider is when $\ystart = y_t \text{ and } \mstart < 1$. In this case we have 
\begin{align*}
    & \Eb{(1 -  \mstart)^2\frac{K-1}{K} - (1 -  \mstart)^2 + \frac{\eta}{p_t'(y_t)} 4 \|\x_t\|^2(1 -  \mstart)^2} \\
    & = \Eb{-(1 -  \mstart)^2\frac{1}{K} + \frac{\eta}{p_t'(y_t)} 4 \|\x_t\|^2(1 -  \mstart)^2}\\
    & = \Eb{-(1 -  \mstart)^2\frac{1}{K} + \frac{K \eta}{\gamma} 4 X^2(1 -  \mstart)^2},
\end{align*}
which is bounded by 0 since $\eta = \frac{\gamma}{4 K^2 X^2}$.

Finally, if $\ystart \not = y_t \text{ and } \mstart \geq 1$ then 
\begin{align*}
    \Eb{1 - (1 - 2 m_t(\W_t, y_t)) + \frac{\eta}{p_t'(y_t)} 4 \|\x_t\|^2} = & \Eb{2 m_t(\W_t, y_t) + \frac{\eta}{p_t'(y_t)} 4 \|\x_t\|^2} \\
    \leq & \Eb{-2\mstart + \frac{\eta}{p_t'(y_t)} 4\|\x_t\|^2} \tag{by equation \eqref{eq:mtstar + mt}}\\
    \leq & -2 + \frac{K \eta}{\gamma} 4X^2,
\end{align*}
which is bounded by 0 since $\eta < \frac{\gamma}{2K^2X^2}$. We apply Lemma \ref{lem: surrogate gap} and use the above to find:
\begin{align*}
    \E\left[\sumT  \id[y_t' \not = y_t]\right] \leq & \frac{\|\U\|^2}{2\eta} + \Eb{\sumT\ell_t(\U)} + \gamma T  \\
     & + \Eb{\sumT \left((1 - a_t)\id[\ystart \not = y_t] + a_t \frac{K-1}{K} - \ell_t(\W_t) + \frac{\eta}{2} \|\g_t\|^2\right)} \\
     \leq & \frac{D^2}{2\eta} + \gamma T + \Eb{\sumT \ell_t(\U)}.
\end{align*}
Plugging in $\eta = \frac{\gamma}{4 K^2 X^2}$ gives us:
\begin{align*}
    \E\left[\sumT  \id[y_t' \not = y_t]\right] \leq & \frac{2K^2X^2D^2}{\gamma}  + \gamma T + \Eb{\sumT \ell_t(\U)}.
\end{align*}
Now we set $\gamma = \min\left\{1, \sqrt{\frac{2K^2X^2D^2}{T}}\right\}$. In the case where $1 \leq \sqrt{\frac{2K^2X^2D^2}{T}}$ we have 
\begin{align*}
    \E\left[\sumT  \id[y_t' \not = y_t]\right] \leq & 4K^2X^2D^2 + \Eb{\sumT \ell_t(\U)}.
\end{align*}
In the case where $1 > \sqrt{\frac{2K^2X^2D^2}{T}}$ we have 
\begin{align*}
    \E\left[\sumT  \id[y_t' \not = y_t]\right] \leq & 2 DKX\sqrt{2T} + \Eb{\sumT \ell_t(\U)},
\end{align*}
which completes the proof.

\end{proof}

\section{Online Classification with Abstention}\label{app:adahedge with abstentions}

\begin{algorithm}[t]
\caption{\textproc{AdaHedge} with abstention}\label{alg:AdaAb}
\begin{algorithmic}[1]
\Require \textproc{AdaHedge} 
\For{$t = 1 \ldots T$}
\State Obtain expert predictions $\y_t = (y_t^1, \ldots, y_t^d)^\top \in [-1, 1]^d$ 
\State Obtain expert distribution $\hat{\p}_t$ from \textproc{AdaHedge}
\State Set $\yhatt = \inner{\hat{\p}_t}{\y_t}$
\State Let $\ystart = \sign(\yhatt)$
\State Set $b_t = 1 - |\yhatt|$
\State Predict $y_t' = \ystart$ with probability $1 - b_t$ and predict $y_t'  = * $ with probability $b_t$
\State Obtain $\ell_t$ and send $\ell_t$ to \textproc{AdaHedge}
\EndFor
\end{algorithmic}
\end{algorithm}

The online classification with abstention setting was introduced by \citet{neu2020fast} and is a special case of the prediction with expert advice setting \citet{vovk1990aggregating, LittlestoneWarmuth1994}. For brevity we only consider the case where there are only 2 labels, -1 and 1. The online classification with abstention setting is different from the standard classification setting in that the learner has access to a third option, abstaining. \citet{neu2020fast} show that when the cost for abstaining is smaller than $\half$ in all rounds it is possible to tune Exponential Weights such that it suffers constant regret with respect to the best expert in hindsight. \citet{neu2020fast} only consider the zero-one loss, but we show that a similar bound also holds for the hinge loss (and also for the zero-one loss as a special case of the hinge loss). We use a different proof technique from \citet{neu2020fast}, which was the inspiration for the proofs of the mistake bounds of \textproc{Gaptron}. Instead of vanilla Exponential Weights we use a slight adaptation of \textproc{AdaHedge} \citep{deRooij2014} to prove constant regret bounds when all abstention costs $c_t$ are smaller than $\half$. In online classification with abstention, in each round $t$
\begin{itemize}
    \item[1] the learner observes the predictions $y_t^i \in [-1, 1]$ of experts $i = 1, \ldots, d$
    \item[2] based on the experts' predictions the learner predicts $y_t'\in [-1, 1] \cup *$, where $*$ stands for abstaining
    \item[3] the environment reveals $y_t \in \{-1, 1\}$
    \item[4] the learner suffers loss $\ell_t(y_t') = \half (1 - y_t y_t')$ if $y_t' \in [-1, 1]$ and $c_t$ otherwise. 
\end{itemize}

The algorithm we use can be found in Algorithm \ref{alg:AdaAb}. A parallel result to Lemma \ref{lem: surrogate gap} can be found in Lemma \ref{lem:abstention gap}, which we will use to derive the regret of Algorithm \ref{alg:AdaAb}.
\begin{lemma}\label{lem:abstention gap}
For any expert $i$, the expected loss of Algorithm \ref{alg:AdaAb} satisfies:
\begin{align*}
    \sumT \left((1-b_t)\ell_t(\ystart) + b_t c_t\right) \leq & \sumT \ell_t(y_t^i) + \inf_{\eta > 0}\left\{\frac{\ln(d)}{\eta} + \sumT \underbrace{\left((1 - b_t)\ell_t(\ystart) + c_t b_t + \eta v_t - \ell_t(\yhatt)\right)}_{\textnormal{Abstention gap}}\right\} \\
    & + \frac{4}{3} \ln(d) + 2,
\end{align*}
where $v_t = \E_{i\sim \hat{\p}_t}[(\ell_t(\yhatt) - \ell_t(y_t^i))^2]$.
\end{lemma}
Before we prove Lemma \ref{lem:abstention gap} let us compare Algorithm \ref{alg:AdaAb} with \textproc{Gaptron}. The updates of weight matrix $\W_t$ in \textproc{Gaptron} are performed with OGD. In Algorithm \ref{alg:AdaAb} the updates or $\hat{p}_t$ are performed using \textproc{AdaHedge}. The roles of $a_t$ in \textproc{Gaptron} and $b_t$ in Algorithm \ref{alg:AdaAb} are similar. The role of $a_t$ is to ensure that the surrogate gap is bounded by $0$, the role of $b_t$ is to ensure that the abstention gap is bounded by 0. 
\begin{proof}[Proof of Lemma \ref{lem:abstention gap}]
First, \textproc{AdaHedge} guarantees that
\begin{align*}
    \sumT \ell_t(\yhatt) - \ell_t(y_t^i) \leq 2\sqrt{\ln(d)\sumT v_t} + 4/3 \ln(d) + 2.
\end{align*}
Using the regret bound of \textproc{AdaHedge} we can upper bound the expectation of the loss of the learner as 
\begin{align*}
    & \sumT \left((1-b_t)\ell_t(\ystart) + b_t c_t\right) \\
    & = \sumT \left((1-b_t)\ell_t(\ystart) + b_t c_t + \ell_t(y_t^i) - \ell_t(\yhatt)\right) + \sumT\left( \ell_t(\yhatt)-\ell_t(y_t^i)\right) \\
    & \leq \sumT \left((1-b_t)\ell_t(\ystart) + b_t c_t + \ell_t(y_t^i) - \ell_t(\yhatt)\right) + 2\sqrt{\ln(d)\sumT v_t} + 4/3 \ln(d) + 2 \\
    & = \sumT \ell_t(y_t^i) + \inf_{\eta > 0}\left\{\frac{\ln(d)}{\eta} + \sumT \left((1 - b_t)\ell_t(\ystart) + c_t b_t + \eta v_t - \ell_t(\yhatt)\right)\right\}  + 4/3 \ln(d) + 2.
\end{align*}
\end{proof}

To upper bound the abstention gap by 0 is more difficult than to upper bound the surrogate gap as the negative term is no longer an upper bound on the zero-one loss. Hence, the abstention cost has to be strictly better than randomly guessing as otherwise there is no $\eta$ or $b_t$ such that the abstention gap is smaller than 0. The result for abstention can be found in Theorem \ref{th:abstention} below.

\begin{theorem}\label{th:abstention}
Suppose $\max_t c_t < \half$ for all $T$. Then Algorithm \ref{alg:AdaAb} guarantees
\begin{equation*}
    \sumT \left((1-b_t)\ell_t(\ystart) + b_t c_t\right) \leq \sumT \ell_t(y_t^i) + \min\left\{\frac{\ln(d)}{1-2\max_t c_t}, 2\sqrt{\ln(d)\sumT v_t}\right\} + 4/3 \ln(d) + 2.
\end{equation*}
\end{theorem}
\begin{proof}
We start by upper bounding the $v_t$ term. We have
\begin{align*}
    v_t = \frac{1}{4} \E_{\hat{\p}_t}\left[(y_t^i - \yhatt)^2\right] \leq \frac{1}{4} (1 - \yhatt)(\yhatt + 1) \leq \half (1 - |\yhatt|),
\end{align*}
where the first inequality is the Bhatia-Davis inequality \citep{bhatia2000better}. 
As with the proofs of \textproc{Gaptron} we split the abstention gap in cases:
\begin{equation}
\begin{split}
    & (1 - b_t)\ell_t(\ystart) + c_t b_t + \eta v_t - \ell_t(\yhatt) \\
    & \leq (1 - b_t)\ell_t(\ystart) + c_t b_t + \eta \half (1 - |\yhatt|) - \ell_t(\yhatt) \\
    & = \begin{cases}
        c_t (1 - |\yhatt|) + \eta \half (1 - |\yhatt|) - \half (1 - |\yhatt|) & \text{ if } \ystart = y_t \\
        |\yhatt| + c_t (1 - |\yhatt|) + \eta \half (1 - |\yhatt|) - \half(1 + |\yhatt|) & \text{ if } \ystart \not = y_t.
    \end{cases}
\end{split}
\end{equation}
Note that regardless of the true label $(1 - b_t)\ell_t(\ystart) + c_t b_t - \ell_t(\yhatt) \leq 0$ since $c_t < \half$. Hence, by using Lemma \ref{lem:abstention gap}, we can see that as long as $c_t < \half$
\begin{equation*}
    \sumT (1-b_t)\ell_t(\ystart) + b_t c_t \leq \sumT \ell_t(y_t^i) + 2\sqrt{\ln(d)\sumT v_t} + 4/3 \ln(d) + 2.
\end{equation*}

Now consider the case where $\ystart = y_t$. In this case, as long as $\eta \leq 1 - 2 c_t$ the abstention gap is bounded by 0. If $\ystart \not = y_t$ then 
\begin{align*}
    |\yhatt| + c_t (1 - |\yhatt|) + \eta \half (1 - |\yhatt|) - \half(1 + |\yhatt|) = c_t (1 - |\yhatt|) + \eta \half (1 - |\yhatt|) - \half(1 - |\yhatt|).
\end{align*}
So as long as $\eta \leq 1 - 2 c_t$ the abstention gap is bounded by 0. Applying Lemma \ref{lem:abstention gap} now gives us
\begin{align*}
    \sumT (1-b_t)\ell_t(\ystart) + b_t c_t - \ell_t(y_t^i) \leq & \inf_{\eta > 0} \left\{\frac{\ln(d)}{\eta}  + \sumT \left((1 - b_t)\ell_t(\ystart) + c_t b_t + \eta v_t - \ell_t(\yhatt)\right) \right\}\\
    & + 4/3 \ln(d) + 2 \\
    \leq & \frac{\ln(d)}{1-2\max_t c_t} + 4/3 \ln(d) + 2,
\end{align*}
which completes the proof.
\end{proof}

With a slight modification of the proof of Theorem \ref{th:abstention} one can also show a similar result as Theorem 8 by \citet{neu2020fast}, albeit with slightly worse constants. We leave this as an exercise for the reader. 

\end{document}